\documentclass[letterpaper]{article}
\usepackage{aaai25} 
\usepackage{times} 
\usepackage{helvet} 
\usepackage{courier} 
\usepackage[hyphens]{url} 
\usepackage{graphicx} 
\usepackage{graphicx}  
\usepackage{natbib}  
\usepackage{caption}  
\usepackage{algorithm}
\usepackage{algpseudocode}
\usepackage{color}
\usepackage{amsmath} 
\usepackage{amsthm}
\usepackage{amssymb}
\usepackage{comment}
\usepackage{caption}
\usepackage{subcaption}

\newcommand{\distr}[1]{\mathbb D(#1)}

\newcommand{\prob}[3]{\text{prob}_{#1}^{#2}\left(#3\right)}
\newcommand{\Shielded}[1]{\text{Sh}^{\leq p}_{\beta}\left(#1\right)}
\newcommand{\Shieldedtwo}[2]{\text{Sh}^{\leq p}_{#2}\left(#1\right)}
\newcommand{\first}[1]{\text{first}\left(#1\right)}
\newcommand{\last}[1]{\text{last}\left(#1\right)}
\newcommand{\addmemory}[1]{\widehat{#1}}

\newcommand{\paths}[1]{\text{Paths}\left(#1\right)}
\newcommand{\expect}[4]{\mathbb E^{#3}_{#1,#2}\left(#4\right)}
\newcommand{\Bellman}[1]{\mathcal B_{#1}}
\newcommand{\encoded}[1]{\text{En}_{\beta}^{\leq p}(#1)}
\newcommand{\Shielding}[1]{\overline{#1}}
\newcommand{\setbetas}[1]{\mathcal B\left(#1\right)}
\newcommand{\infnorm}[1]{\left\lVert#1\right\rVert_{\infty}}

\newcommand{\reach}[1]{\mbox{Reach}\left(#1\right)}

\newtheorem{theorem}{Theorem}

\newtheorem{definition}{Definition}
\newtheorem{assumption}{Assumption}
\newtheorem{corollary}{Corollary}
\usepackage{newfloat}
\usepackage{listings}
\DeclareCaptionStyle{ruled}{labelfont=normalfont,labelsep=colon,strut=off} 
\floatstyle{ruled}
\newfloat{listing}{tb}{lst}{}
\floatname{listing}{Listing}
\title{Probabilistic Shielding for Safe Reinforcement Learning}
\author{
    Edwin Hamel-De le Court,
    Francesco Belardinelli,
    Alexander W. Goodall\\
}
\affiliations{Imperial College London\\
    \{e.hamel-de-le-court, francesco.belardinelli, a.goodall22\}@ic.ac.uk

}

\usepackage{bibentry}

\begin{document}

\maketitle

\begin{abstract}
In real-life scenarios, a Reinforcement Learning (RL) agent aiming to maximise their reward, must often also behave in a safe manner, including at training time. Thus, much attention in recent years has been given to Safe RL, where an agent aims to learn an optimal policy among all policies that satisfy a given safety constraint. However, strict safety guarantees are often provided through approaches based on linear programming, and thus have limited scaling. In this paper we present a new, scalable method, which enjoys strict formal guarantees for Safe RL, in the case where the safety dynamics of the Markov Decision Process (MDP) are known, and safety is defined as an undiscounted probabilistic avoidance property. Our approach is based on state-augmentation of the MDP, and on the design of a shield that restricts the actions available to the agent. We show that our approach provides a strict formal safety guarantee  that the agent stays safe at training and test time. Furthermore, we demonstrate that our approach is viable in practice through experimental evaluation.
\end{abstract}

\section{Introduction}
Reinforcement Learning (RL) aims to optimise the behaviour of an agent in an unknown environment. Much attention has been given to RL in recent years, because of its many practical applications, which include for example playing games \cite{DBLP:journals/corr/MnihKSGAWR13} or robotics \cite{SurveyRobotics}. In many of these applications however, safety, either for the agent or for other humans, is critical. Consequently, Safe RL has been developed, where an agent has to maximise its cumulative expected reward subject to a constraint. Even though much progress has been made to provide safety guarantees during training and at test time, approaches providing strict safety guarantees still rely on Linear Programming \cite{10.1145/800057.808695}, which is known to lack scalability.

\paragraph{Contributions.}
We design a new shielding approach for finding a policy that maximizes cumulative reward in a finite MDP with known safety dynamics while guaranteeing safety throughout the whole learning phase. We consider MDPs where some states are labelled unsafe, and the safety we consider consists in avoiding those unsafe states with at least some probability $p$. This framework comes from a probabilistic version of what is usually called the ``safety" fragment of Linear Temporal Logic (LTL) \cite{ABENTShielding} \cite{DBLP:conf/concur/0001KJSB20}. Although, for simplicity's sake, we consider only a subset of that safety fragment, \emph{i.e.} safety definable by state-avoidance, it is possible to reduce the whole fragment to that subset with the usual trick of making a product between the automaton representing the LTL property and the MDP \cite{ABENTShielding} \cite{DBLP:conf/concur/0001KJSB20}. Thus, the model we use can capture many real-life scenarios, like a robot's task of reaching a goal position while avoiding objects on his path.

Our approach is, to our knowledge, in the framework we consider, the first approach not based on Linear Programming \cite{10.1145/800057.808695} that gives \emph{strict formal guarantees} for safety throughout learning and at test time. Instead of using Linear Programming, that has limited scalability \cite{10.5555/3312046}, we leverage sound Value Iteration algorithms (see e.g. \cite{HADDAD2018111,DBLP:conf/cav/QuatmannK18,DBLP:conf/cav/HartmannsK20}), which are algorithms that improve on Value Iteration in order to give formal approximation guarantees.
We use the values obtained by Value Iteration over safety costs to construct a shield that makes the agent's exploration of the MDP safe by constraining its actions. Constructing a shield is a well-known approach to solve constrained RL \cite{ABENTShielding,DBLP:conf/concur/0001KJSB20,DBLP:conf/atal/Elsayed-AlyBAET21,DBLP:conf/ijcai/YangMRR23}. However, in contrast to shields defined in previous papers, the shield we construct does not directly constrain the actions of the agent in the original MDP, but constrains the actions of the agent in a safety-aware state-augmented MDP. This approach allows for preserving optimality while ensuring safety in a probabilistic context. Any RL algorithm can then be used to solve the shielded MDP, such as PPO \cite{SPPO}, A2C \cite{DBLP:journals/corr/MnihBMGLHSK16}, etc. Once the shield is constructed, it is used to train the agent to maximize its cumulative reward, with no constraint violations incurred.
To summarize, the contributions of the paper are as follows.
\begin{enumerate}
    \item We design a shield for a finite MDP as a safety-aware state-augmention of that MDP using only its safety dynamics.
    \item We show that the shield makes the agent's exploration safe.
    \item We show that finding an optimal policy among all safe policies reduces to finding an optimal policy in the shield.
    \item We provide a practical way of implementing a shield as a gym environment.
\end{enumerate}
The paper is organized as follows. The preliminaries introduce the mathematical background and notations needed for the paper. Then, we introduce the problem considered, give an overview of our approach, formally define the shield and show that it is safe and optimality-preserving. Finally, we discuss a practical way to implement the shield.

\subsection{Related Work}
Safe Reinforcement Learning has gathered much attention in recent years, and several approaches have been proposed. A comprehensive survey can be found in \cite{DBLP:journals/corr/abs-2205-10330}. Policy-based approaches are arguably the most popular approaches in Safe RL. They usually consist in extending a known RL algorithm like PPO \cite{SPPO}, TRPO \cite{DBLP:journals/corr/SchulmanLMJA15}, or SAC \cite{HSAC} to Safe RL using a lagrangian \cite{Achiam2019BenchmarkingSE}, or in enforcing the constraint by changing the objective function \cite{DBLP:conf/aaai/LiuDL20} or modifying the update process \cite{DBLP:journals/corr/abs-2002-06506}, without introducing a dual variable. Some of these algorithms have become widely used for benchmarking against other new Safe RL algorithms, and are implemented in several state-of-the-art Safe RL frameworks \cite{Achiam2019BenchmarkingSE,ji2023safetygymnasiumunifiedsafereinforcement,ji2023omnisafeinfrastructureacceleratingsafe}. We introduce in the following other, more specific, approaches that relate to ours.

\paragraph{Shielding.}
A \emph{shield} is a system that restricts the action of the agent during learning and at test time, to ensure its safety \cite{ABENTShielding,DBLP:conf/concur/0001KJSB20,DBLP:conf/atal/Elsayed-AlyBAET21,DBLP:conf/ijcai/YangMRR23}. Shielding was introduced in \cite{ABENTShielding}, where safety was defined as a formula of the "safety" fragment of LTL. This paper can be considered an extension of \cite{ABENTShielding} to the probabilistic setting, since in the case where the agent must be safe with probability $1$, the shield we introduce is almost the same as the one defined in \cite{ABENTShielding}. In \cite{DBLP:conf/concur/0001KJSB20}, the authors already also extend \cite{ABENTShielding} to the probabilistic case, but no formal guarantees are provided for the safety of the policy, and our respective methods are significantly different. In particular, their approach may indeed violate safety in practice. A comprehensive survey focused on shielding can be found in \cite{10039301}.

\paragraph{Linear Programming-based approaches.}
It is well-known that Safe RL can be solved by Linear Programming under certain assumptions \cite{Altman1999ConstrainedMD}. Several recent works have leveraged Linear Programming to provide statistical guarantees in the model-based case, when the dynamics of the MDP are learned in various contexts. For example, \cite{DBLP:journals/corr/abs-2003-02189} and \cite{DBLP:conf/nips/BuraHKSC22} provide algorithms for discounted cumulative rewards and costs, while \cite{DBLP:journals/corr/abs-2403-15928} provides an algorithm in the case where safety is defined by a reach-avoid undiscounted property. 

\paragraph{State augmentation techniques.}
State-augmentation techniques with a number representing how far the agent is from being unsafe have also been studied. In \cite{DBLP:journals/corr/abs-2102-11941}, every state of the MDP is augmented with a Lagrange multiplier. In \cite{DBLP:conf/icml/SootlaCJWMWA22}, the states are augmented with the accumulated safety cost up to that state, and are used to reshape the objective. In \cite{DBLP:conf/atal/YangJTF24}, states are also augmented with the accumulated safety cost up to that state, and non-stationary policies depending on that cost are considered. However, in contrast to the aforementioned papers, the agent in our approach \emph{is able to choose} the maximal accumulated safety cost he is able to use in the future depending on the state it goes to. Thus, we do not need to use the augmentation to change the objective function, and are able to provide stricter optimality-preserving and safety guarantees.

\section{Preliminaries}
We introduce in this section the mathematical prerequisites necessary for the paper.

\subsection{Constrained Reinforcement Learning}
\subsubsection{Markov Decision Processes.}
A \textit{Markov Decision Process (MDP)} is a tuple $\mathcal
M=\langle S,A,P,s_{init}, \allowbreak AP, L, R\rangle$, where $S$ is a
set of \textit{states}; $A$ is a mapping that associates every state
$s\in S$ to a nonempty finite set of \emph{actions} $A(s)$; $P$ is a
\textit{transition probability function} that maps every state-action
pair $(s,a)$ to a probability
measure over $S$; $s_{init}\in S$ is the \textit{initial state}\footnote{This can be assumed wlog compared to a model with an \emph{initial probability distribution} since it is always possible to add a new initial state to such a model with an action from this initial state whose associated probability distribution is the aforementioned initial probability distribution.};
$AP$ is a set of \emph{atomic propositions} (or atoms); $L:S\mapsto2^{AP}$ is a
\emph{labeling function}; and $R:S\mapsto \mathbb R$ is the
\emph{reward function}. For the sake of simplicity, we may write
$P(s,a,s')$ instead of $P(s,a)(s')$. An MDP is finite if the sets of
states and actions are finite.

A finite (resp.~infinite) {\em path} in $\mathcal M$ is a finite
(resp.~infinite) word $\zeta=s_0 a_0 \cdots s_{n-1}a_{n-1}s_n$
(resp.~$\zeta=s_0 a_0\cdots s_n a_n\cdots$) such that
$s_0=s_{init}$, and such that for any positive integer $i\leq n$ (resp. for any positive integer $i$), $s_i$ is a state of $\mathcal M$, $a_{i-1}$ is an action in $A(s_{i-1})$, and $s_i$ is in the support of $P(s_{i-1},a_{i-1})$. In addition, for any finite path $\zeta=s_0 a_0 \cdots s_{n-1}a_{n-1}s_n$ of $\mathcal M$, we let
$\first{\zeta}=s_0$, we let $\last{\zeta}=s_n$, and we let $\paths{\mathcal M}$ denote the set
of {\em infinite} paths of $\mathcal M$. A {\em policy} $\pi$ of $\mathcal M$ is a
mapping that associates any finite path $\zeta$
of $\mathcal M$ to an element of $\distr{A(\last{\zeta})}$, where $\distr{E}$ is the set of all probability measures over $E$. It is memoryless if $\pi(\zeta)$ only depends on $\last{\zeta}$. It is deterministic if for any finite path $\zeta$ of $\mathcal M$, $\pi(\zeta)$ is a Dirac measure. For any policy
$\pi$ of $\mathcal M$, for any state $s\in S$, we let $\mathcal
M^s_\pi$ denote the {\em Markov chain} induced by $\pi$ in $\mathcal M$
starting from state $s$, we let $\mathcal M_\pi$ denote
$\mathcal M^{s_{init}}_\pi$, and we let $P_\pi$ denote the transition
function of $\mathcal M_\pi$. We denote the usual probability measure
induced by the Markov chain $\mathcal M^s_\pi$ on $\paths{\mathcal M}$
by $\text{prob}^{s}_{\mathcal M,\pi}$. For more details on MDPs and induced Markov chains, see \cite{baier2008principles,BSSstochastic}.

\paragraph{Reinforcement Learning.}

For any random variable
$X:\paths {\mathcal M}\mapsto \mathbb R$, we let $\expect{\mathcal
  M}{\pi}{s}{X}$ denote the expectation of $X$ with respect to the
probability measure $\text{prob}^{s}_{\mathcal M,\pi}$, and we let
$\expect{\mathcal M}{\pi}{}{X}$ denote $\expect{\mathcal
  M}{\pi}{s_{init}}{X}$. In addition, when there is no ambiguity, we
usually write $\expect{\mathcal M}{\pi}{s}{\bullet}$ for
$\expect{\mathcal M}{\pi}{s}{s_0 a_0\cdots s_{n} a_n \cdots \mapsto \bullet}$.

In Reinforcement Learning, we usually solve the following the problem: given an MDP $\mathcal M$, and a discount factor $0<\gamma<1$, find a policy $\pi^\star$ such that $J(\pi^\star)=\max_{\pi}(J(\pi))$, where $$J(\pi)=\expect{\mathcal M}{\pi}{s}{\sum_{t\in\mathbb N}
  \gamma^t R(s_t)}.$$
In recent years, many algorithms have been proposed to solve the above problem. Proximal Policy Optimization (PPO) \cite{SPPO}, Asynchronous Advantage Actor Critic (A3C) \cite{DBLP:journals/corr/MnihBMGLHSK16}, or Soft Actor-Critic (SAC) \cite{HSAC}, are among the most popular, and we leverage these algorithms in our experiments.

\paragraph{Probabilistic Reachability Goals.}

In contrast to discounted objectives, reachability goals are a simple form of undiscounted and infinite horizon objectives, that we use as a constraint for MDPs. For any MDP $\mathcal M$, any state $s$ and policy $\pi$ of $\mathcal M$, any finite path $\zeta=s_0 a_0\cdots s_{n-1} a_{n-1} s_n$ (resp. infinite path $\zeta=s_0 a_0\cdots s_{n} a_n \cdots$) in $\mathcal M$, we write $\zeta\models\reach{\textbf c}$ if there exists $i\in\leq n$ (resp. $i\in\mathbb N$) such that $L(s_i)=\textbf c$. Then, $\mathcal M_\pi\models \mathbb P_{\leq p}\left(\reach{ \textbf c}\right)$ when the property $\text{prob}^{s}_{\mathcal M,\pi}\{\zeta\in\paths{\mathcal M^s_\pi} \mid \zeta\models\reach{\textbf c}\}\leq p$ is true.

\section{Probabilistic Shielding}
In this section, we introduce a new theoretical framework for probabilistic shielding, and show that it gives safety and optimality guarantees.

\subsection{Problem Statement}

We assume in the rest of the paper that all the labelling functions of the MDPs considered take values in $\{\textbf s,\textbf u\}$, where \emph{safe} states are labelled by $\textbf s$ and \emph{unsafe} states are labelled by $\textbf u$.

\begin{definition}[Reachability-Constrained Optimization Problem (RCOP)] Given a finite MDP $\mathcal M$, a safety threshold $0\leq p \leq 1$, and a discount factor $0<\gamma<1$, find a policy
$\pi^\star$ such that $\mathcal M_{\pi^\star}\models \mathbb P_{\leq p}(\reach{\textbf u})$
and such that $\pi^\star$ is optimal among all policies $\pi$ satisfying $\mathcal M_\pi\models \mathbb P_{\leq p}(\reach{\textbf u})$, \emph{i.e.,} such that
$$J(\pi^\star)=\max_{\{\pi\mid \mathcal M_{\pi}\models \mathbb P_{\leq p}\left(\reach{\textbf u}\right)\}}(J(\pi)).$$
\end{definition}
The above problem is a form of generalisation of the problem considered in \cite{ABENTShielding}. Furthermore, variants of the above problem, \emph{i.e} RL with infinite-horizon undiscounted safety properties, are considered in many books and papers (see for example \cite{BSSstochastic, Altman1999ConstrainedMD, DBLP:conf/concur/0001KJSB20, DBLP:journals/corr/abs-2403-15928, DBLP:conf/atal/MqirmiBL21}). However, a significant difference between the problem we consider and most Safe RL approaches is that, similarly to \cite{DBLP:conf/atal/YangJTF24}, we make the choice of including non-stationary policies in the problem. We make this choice because in our context, where the discount factor of the reward (which is less than $1$) and the discount factor of the constraint cost (which is equal to $1$) are not equal, optimal memoryless policies are not guaranteed to exist \cite{Altman1999ConstrainedMD}.

\subsection{Method Overview}


In order to tackle RCOP, we compute, for all states of the MDP, an approximation of the minimal probability of reaching, from that state, an unsafe state. More precisely, we define $\beta_{\mathcal M}$ as the mapping such that for every state $s$ of the MDP $\mathcal M$, $\beta_{\mathcal M}(s)$ is equal to $$\min_{\pi} \text{prob}^{s}_{\mathcal M,\pi}\{\zeta\in\paths{\mathcal M^s_\pi} \mid \zeta\models \reach{\textbf u}\}.$$ 
The mapping $\beta_{\mathcal M}$ is the smallest fixed point of the following equation \cite{baier2008principles},
$$\beta(s)=\left\{\begin{array}{lc}
    1 & \text{if } L(s)=\textbf u \\
    \left(\Bellman{\mathcal M}(\beta)\right)(s)& \text{otherwise,}
\end{array} \right. $$
where $\left(\Bellman{\mathcal M}(\beta)\right)(s)=\min_{a\in A} \sum_{s'\in S} P(s,a,s') \beta(s')$.

This fixed point can be computed with linear programming
\cite{Forejt2011} in 
polynomial time 
\cite{10.1145/800057.808695}. 
In
practice, this approach is inefficient and state-of-the-art methods
rely on \emph{value iteration} (VI), i.e., iterating the operator
$\Bellman{\mathcal M}$ from $\beta_0$ such
that $\beta_0(s)=1$ if $L(s)=\textbf u$, and $\beta_0(s)=0$ otherwise
\cite{10.5555/3312046} to compute an approximation of $\beta_{\mathcal M}$. However, VI might not yield a good approximation of $\beta_{\mathcal M}$ if stopped prematurely and only gives a
lower bound on $\beta_{\mathcal M}$,
whereas an upper bound is needed to
provide safety guarantees in our approach.
\begin{definition}
    For any MDP $\mathcal M$, and any $\epsilon\geq 0$, an \emph{inductive $\epsilon$-upper bound} of $\beta_{\mathcal M}$ is a mapping $\beta$ that associates to any state $s$ of $\mathcal M$ a number in $[0;1]$ such that for all states $s$, $0\leq \beta(s)-\beta_{\mathcal M}(s)\leq \epsilon$, and $\left(\Bellman{\mathcal M}(\beta)\right)(s)\leq \beta(s)$.
\end{definition}

The first step of our approach thus consists in computing an inductive $\epsilon$-upper bound $\beta$ of $\beta_{\mathcal M}$, with a small $\epsilon$. To our knowledge, the fastest algorithms for that purpose in the general case are Interval Iteration \cite{HADDAD2018111}, Sound Value Iteration \cite{DBLP:conf/cav/QuatmannK18} and Optimistic Value Iteration \cite{DBLP:conf/cav/HartmannsK20}, with no clear overall faster one \cite{DBLP:conf/cav/HartmannsK20}.

Once $\beta$ is computed, we construct a shield $\Shielded{\mathcal M}$ by augmenting every state of the MDP $\mathcal M$ with a real number in $[0;1]$, that is, a ``safety level" representing intuitively a maximal probability of reaching an unsafe state from the current state while following \emph{any} 
actions. Thus, any action $(a,\alpha)$ taken in $\Shielded{\mathcal M}$ is composed of an action $a$ of $\mathcal M$, together with predictions $\alpha \in [0;1]$, that may depend on the current ``safety level", as to what the next ``safety levels" will be. Furthermore, $\Shielded{\mathcal M}$ is defined so that:
\begin{enumerate}
    \item The predictions must be coherent, \emph{i.e} that the sum of the next ``safety levels" as predicted by $\alpha$, pondered with the probabilities given by action $a$, is less than or equal to the current ``safety level". \label{condition-beta}
    \item The predicted ``safety levels" cannot be less than the inductive $\epsilon$-upper bound $\beta$ of the minimal probability $\beta_{\mathcal M}(s)$ of reaching in $\mathcal M$ an unsafe state from the current state $s$ \label{safety-level-condition}.
\end{enumerate}

Finally, we learn a policy in the constructed shield. 
\subsection{The Shield: Safety and Optimality Guarantees}
We now give a formal definition of the shield used in our approach, and justify that our approach is theoretically sound. In the following, we let $\mathcal M$ be an MDP,  $\gamma\in [0;1]$ be a discount factor, $p\in [0;1]$, $\epsilon\in\mathbb R^+$, and we let $\beta$ be an inductive $\epsilon$-upper bound of $\beta_{\mathcal M}$ such that $\beta(s_{init}) \leq p$. Notice that if such a $\beta$ does not exist, RCOP is unfeasible.
Moreover, for any $s\in S$, we 
let $\chi^s$ denote the polytope in $\mathbb R^{A(s)}$ representing probability distributions over $A(s)$, \emph{i.e.} the set of all $x\in \mathbb R^{A(s)}$ such that $\displaystyle\sum_{a\in A(s)} x_a =1$ and $x_a\geq 0$ for any $a\in A(s)$, and for any $s\in S$, for any mapping $\alpha$ from $S$ to $[0;1]$, any $q \in [\beta(s);1]$, we let $A^{s,q}_\alpha$ denote the half-space representing Condition \ref{safety-level-condition} above with 
$s$ being the current state and $q$ being the current``safety level", \emph{i.e.} we let $A^{s,q}_\alpha$ denote the set of all $x\in \mathbb R^{A(s)}$ such that $$\displaystyle\sum_{a\in A(s)} x_a \left(q-\sum_{s'\in S}P(s,a,s')\alpha(s')\right)\geq 0.$$
Finally, we let $C_\alpha^{s,q}$ be the polytope of probability distributions satisfying Condition \ref{safety-level-condition}, \emph{i.e.} the polytope defined by $\chi^s\cap A^{s,q}_\alpha$, we let $V^{s,q}_\alpha$ be the (finite) set of vertices of $C_\alpha^{s,q}$, and we let $X=\prod_{s\in S}[\beta(s);1]$.

\begin{definition}[The Shield]\label{def-shield}
We let $\Shielded{\mathcal M}$ be the MDP $\mathcal M'$ with
\begin{itemize}
\item set of states $S'= \{(s, q) \mid s \in S, q \in [\beta(s);1]\}$;
\item sets of actions $A'(s,q)=\bigcup_{\alpha\in X}\bigcup_{v\in V^{s,q}_\alpha} (\alpha,v)$;
\item initial state $s'_{init}=(s_{init},p)$;
\item labelling $L'(s,x)=L(s)$;
\item reward $R'(s,x)=R(s)$;
\item transition probability function $P'$ such that for any $(s,q)\in S'$, any $\alpha\in X$, and any $v\in V^{s,q}_\alpha$, $P'((s,q),(\alpha,v)))$ is equal to $$\sum_{s'\in S} \delta_{(s',\alpha(s))}\left(\sum_{a\in A(s)} v_a P(s,a,s') \right)$$
where $\delta_{(s',\alpha(s'))}$ is the Dirac measure on $S'$ with support $\{(s',\alpha(s'))\}$.
\end{itemize}
\end{definition}

In the above definition, $\alpha$ corresponds to the "safety levels" of the next states chosen by the agent, the definition of $X$ guarantees that Condition \ref{safety-level-condition} is satisfied, and the polytope $C^{s,q}_\alpha$ corresponds to all combinations of actions that satisfy Condition \ref{condition-beta}.
Notice that $\Shielded{\mathcal M}$ is indeed an MDP since for every $(s,q)\in S'$, $V_\beta^{s,q}$ is nonempty because $\beta$ is inductive, \emph{i.e.} because $\Bellman{\mathcal M}(\beta)\leq \beta$.

We now show that any memoryless policy in the shield is safe. The proof of the following theorem is 
inspired by the proof of Theorem 10.15 in \cite{baier2008principles}.

\begin{theorem}[Safety guarantee in any shield]\label{theorem-safety}
    For any memoryless policy $\pi$ in $\Shielded{\mathcal M}$, we have $$\Shielded{\mathcal M}_{\pi}\models \mathbb P_{\leq p}\left( \reach{\textbf u}\right).$$
\end{theorem}

We now justify that we can use an optimal policy of the shield to find a policy of the original MDP that is safe, and whose expected cumulative reward is close to a solution of RCOP. The closer $\beta$ is to $\beta_{\mathcal M}$, the closer the expected cumulative reward of the policy obtained from our approach will be to the expected cumulative reward of a solution of RCOP. 

For any memoryless policy $\pi$ of $\Shielded{\mathcal M}$, we let $\addmemory{\pi}$ denote the policy of $\mathcal M$ such that $\addmemory{\pi}(s_0\cdots s_n)=\mu_{n}$ where $s_0=s_{init}$, $q_0=p$, $(\alpha^{i+1},v^{i+1})=\pi(s_i,q_i)$, $q_{i+1}=\alpha^{i+1}_{s_{i+1}}$, and $\mu_n(a)=v^{n+1}_a$. It is easy to see that $J(\addmemory{\pi})=J(\pi)$ and that 
\begin{multline*}
\prob{\mathcal M,\addmemory{\pi}}{s_{init}}{\zeta \in \paths{\mathcal M}\mid \zeta\models \reach{\textbf u}}=\\\mbox{prob}_{\Shielded{\mathcal M},\pi}^{(s_{init},p)}\Big(\zeta \in \paths{\Shielded{\mathcal M}}\\\mid \zeta\models \reach{\textbf u}\Big).
\end{multline*}
Thus, as a corollary of Theorem \ref{theorem-safety}, if $\pi$ is a memoryless policy of $\Shielded{\mathcal M}$, $\addmemory{\pi}$ is safe.
\begin{corollary}[Safety guarantee in the original MDP]\label{cor-safety}
    If $\pi$ is a memoryless policy of $\Shielded{\mathcal M}$, then 
    $$\mathcal M_{\addmemory{\pi}}\models \mathbb P_{\leq p}\left( \reach{\textbf u}\right).$$
\end{corollary}
We let $\setbetas{\mathcal M}$ be the set of inductive upper bounds of $\beta_{\mathcal M}$, that we equip with the norm $\infnorm{}$ such that $\infnorm{\beta_1-\beta_2}$ is the maximum of $
|\beta_1(s)-\beta_2(s)|$ for all states $s$ of $\mathcal M$.
\begin{assumption}[Slater's condition]\label{assupmtion-slater}
    There exists a policy $\pi$ in $\mathcal M$ and a number $q<p$ such that 
    $$\mathcal M_{\pi}\models \mathbb P_{\leq q}\left( \reach{\textbf u}\right).$$
\end{assumption}
\begin{theorem}[Optimality-preserving guarantees]\label{theorem-opt}
We have the three following properties.
\begin{enumerate}
    \item For any $\epsilon>0$, for any inductive $\epsilon$-upper bound $\beta$ of $\beta_{\mathcal M}$, there exists an optimal, memoryless, and deterministic policy $\pi^\star_\beta$ of $\Shielded{\mathcal M}$.
    \item The policy $\addmemory{\pi^\star_{\beta_{\mathcal M}}}$ is a solution to RCOP.
    \item If Assumption \ref{assupmtion-slater} holds, then $$\lim_{\beta\in \setbetas{\mathcal M}, \beta\rightarrow \beta_{\mathcal M}} J\left(\addmemory{\pi^\star_{\beta}}\right)=J\left(\addmemory{\pi^\star_{\beta_{\mathcal M}}}\right).$$
\end{enumerate}
\end{theorem}

\paragraph{Discussion.} 
Definition \ref{def-shield} allows us to construct a shield from any MDP $\mathcal M$ with known safety dynamics via an algorithm that computes an inductive upper bound of $\beta_{\mathcal M}$. Theorem \ref{theorem-safety} and Corollary \ref{cor-safety} show that if we train an agent using the shield, the agent will be safe. Furthermore, Theorem \ref{theorem-opt} justifies that training an agent with the shield yields a cumulative reward close to optimal.

For the sake of simplicity, we made the choice of presenting our shielding approach in the case where the full dynamics of the MDP is known. However, every definition can be straightforwardly adapted to an MDP where only the safety dynamics, \emph{i.e.} a quotient of the MDP containing all of the safety-relevant information, is known. We make use of that adaptation in our experiments. The assumption of knowing the safety dynamics is strong, but is adopted in several papers, and in particular in the majority of shielding methods (see \cite{ABENTShielding,DBLP:conf/atal/Elsayed-AlyBAET21,DBLP:journals/corr/abs-2112-11490} for example), and could be alleviated in the future by introducing a three-step algorithm that at each iteration, learns a better conservative estimation of the safety dynamics, changes $\Shielded{\mathcal M}$ according to that estimation, and does a step of policy iteration in $\Shielded{\mathcal M}$. The size of the state and action space of the shield is bigger than the state and action space of the original MDP, and thus may lead to slower convergence than state-of-the-art Safe RL algorithms. However, the safety of the agent after computing $\beta$ is guaranteed, and the only constraint violations that may thus occur in a real-life scenario occur when computing $\beta$. This is one of the strictest guarantees possible for constraint violations in Safe RL as $\beta$ only depends on the safety dynamics of the MDP, and could be theoretically be computed with any $\epsilon$-greedy safe policy. Thus, if an $\epsilon$-greedy safe policy is known in advance and used to compute $\beta$, the algorithm incurs \emph{exactly zero} constraints violations. This strict guarantee is, to our knowledge, offered in a more scalable way compared to previous Safe RL algorithms that usually use to that end Linear Programming (as in \cite{DBLP:conf/nips/LiuZKKT21} for example).

\begin{algorithm}[tb]
\caption{Probabilistic Shielding}
\label{alg-shielding}
\begin{algorithmic}[1]
\State \textbf{Input:} An MDP $\mathcal M$, a discount factor $\gamma$, an uncertainty threshold $\epsilon$, a safety threshold $p$.
\State Compute an inductive $\epsilon$-upper bound $\beta$ of $$\beta_{\mathcal M}(s)=\max_{\pi} \text{prob}^{s}_{\mathcal M,\pi}\{\zeta \mid \zeta\models \reach{\textbf u}\}$$
\State Construct the shield $\Shielded{\mathcal M}$
\State Learn a memoryless policy $\pi^\star$ in $\Shielded{\mathcal M}$ with an RL algorithm.
\State \textbf{Return} $\addmemory{\pi^\star}$.
\end{algorithmic}
\end{algorithm}

\section{Implementation}
We suppose in the following, without any loss of generality, that for any $s\in S$, there exists an integer $d$ such that $\#A(s)=d$, and we let $\{a^s_1,\ldots,a^s_d\}$ denote $A(s)$. In a gym environment, the policy that the agent follows is output by a neural network. However, even if the sets $A(s)$ all have the same size, this does not guarantee that a probability distribution over a set $A'(s)$ (a set of actions of $\Shielded{\mathcal M}$), can be directly output by a neural network, since the sets $V^{s,q}_\alpha$ do not necessarily all have the same size, even if $s$ and $q$ are fixed. Therefore, to implement the shield as a gym environment, we change the MDP $\Shielded{\mathcal M}$ into an encoded MDP $\encoded{\mathcal M}$ that is equivalent, \emph{i.e} such that every policy of $\encoded{\mathcal M}$ can be transformed into a policy of $\Shielded{\mathcal M}$ and the converse. To avoid instability, the MDP $\encoded{\mathcal M}$ is constructed so that the dependency of the probabilistic transition function on the state-action pair is as continuous as possible. The results obtained show that this approach scales well. For the sake of simplicity, we do not define $\encoded{\mathcal M}$ entirely, but we give in the following the main technical idea of $\encoded{\mathcal M}$, which is a way of mapping the set $V_\alpha^{s,q}$ to a larger set of fixed size, so that the dependency of the probabilistic transition function on the state-action pair is roughly continuous. We give such a mapping $g$ below.

Formally, $g$ associates to any $(s,q,\alpha,i,j)$ such that $s\in S$, $q\in [\beta(s);1]$, $\alpha\in X$, $i,j\in\{1,\ldots,d\}$, and $V^{s,q}_\alpha$ is nonempty, an element of $V^{s,q}_\alpha$. Intuitively, if we let $\chi^s_i$ denote the element of $[0;1]^{A(s)}$ such that ${\chi^s_i}(a)=1$ if $a=a^s_i$ and $0$ otherwise, $g(s,q,\alpha,i,j)$ corresponds to the intersection between the border of the half-space $A_\alpha^{s,q}$ and the line between $\chi^s_i$ and $\chi^s_j$ if there is one, to $\chi^s_i$ if $\chi^s_i$ is in $A_\alpha^{s,q}$, or to a means of the points in $V_\alpha^{s,q}$ weighted by the minimum of their distances to $\chi^s_i$ and $\chi^s_j$ otherwise. A formal definition is given below.

\begin{itemize}
    \item If $i=j$,
    \begin{itemize}
        \item if $\chi^s_i\in A^{s,q}_\alpha$, $g(s,q,\alpha,i,j)=\chi^s_i$,
        \item otherwise $$g(s,q,\alpha,i,j)=\frac{\sum_{v\in V^{s,q}_\alpha}\frac{1}{\lVert \chi^s_i-v\rVert}v}{\sum_{v\in V^{s,q}_\alpha}\frac{1}{\lVert \chi^s_i-v\rVert}},$$
        \end{itemize}
    \item Otherwise, if $i \neq j$,
    \begin{itemize}
        \item if $\chi^s_i\in A^{s,q}_\alpha$, $g(s,q,\alpha,i,j)=\chi^s_i$,
        \item otherwise if $\chi^s_j\in A^{s,q}_\alpha$, $g(s,q,\alpha,i,j)$ is defined as $\lambda_{max}\chi^s_i+(1-\lambda_{max})\chi^s_j$ where $\lambda_{max}$ is the maximal $\lambda\in [0;1]$ such that $\lambda\chi^s_i+(1-\lambda)\chi^s_j\in A^{s,q}_\alpha$ (notice that $g(s,q,\alpha,i,j)\in V_\alpha^{s,q}$ in that case),
        \item and otherwise $$g(s,q,\alpha,i,j)=\frac{\sum_{v\in V^{s,q}_\alpha}\frac{1}{\min(\lVert \chi^s_i-v\rVert,\lVert \chi^s_j-v\rVert)}v}{\sum_{v\in V^{s,q}_\alpha}\frac{1}{\min(\lVert \chi^s_i-v\rVert,\lVert \chi^s_j-v\rVert)}}.$$
    \end{itemize}
    \end{itemize}
Since the convex polytope $C_\alpha^{s,q}$ whose set of vertices is $V_\alpha^{s,q}$ is the intersection of the polytope $\chi$ whose set of vertices is $(\chi_i)_{i\in \{1,\ldots,d\}}$, and of the half-space $A_\alpha^{s,q}$, it is easy to see that the elements of $V_\alpha^{s,q}$ are all on the edges of $\chi$. As a consequence, we have $$V^{s,q}_\alpha\subseteq \{g(s,q,\alpha,i,j)\mid i,j\in\{1,\ldots,d\}\} $$ for any $(s,q)\in S'$ and any $\alpha\in X$, such that $V^{s,q}_\alpha$ is nonempty.

\begin{figure}[!b]
    \begin{subfigure}[t]{0.23\textwidth}
        \centering
        \includegraphics{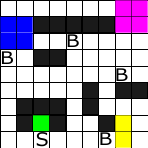}
        \caption{Colour bomb v1}
        \label{fig:9x9colourbomb}
    \end{subfigure}
    \begin{subfigure}[t]{0.23\textwidth}
        \centering
    \includegraphics{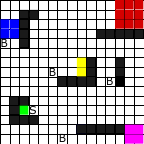}
        \caption{Colour bomb v2}
        \label{fig:15x15colourbomb}
    \end{subfigure}
    \begin{subfigure}[t]{0.23\textwidth}
        \centering
        \includegraphics{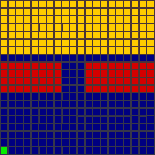}
    \caption{Bridge v1}
    \label{fig:bridgecrossing}
    \end{subfigure}
    \begin{subfigure}[t]{0.23\textwidth}
        \centering
        \includegraphics{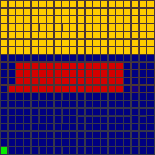}
    \caption{Bridge v2}
    \label{fig:bridgecrossingv2}
    \end{subfigure}
    \caption{Gridworld Environments}
    \label{fig:envs}
\end{figure}

{
\begin{table*}
\begin{tabular}{|l|l|l|l|l|l|l|}
\hline
& \textbf{Media ...} & \textbf{Colour bomb} & \textbf{Color bomb v2} & \textbf{Bridge} & \textbf{Bridge v2} & \textbf{Pacman} \\
\hline
\texttt{random\_action\_probability} & - & 0.1 & 0.1 & 0.04 & 0.04 & -\\
\hline
\texttt{episode\_length} & 40 & 100 & 250 & 600 & 600 & 1000\\
\hline
\texttt{total\_timesteps} & 25k & 25k & 100k & 200k & 200k & 500k\\
\hline
\texttt{safety\_bound} & 0.001 & 0.05 & 0.05 & 0.01 & 0.01 & 0.01\\
\hline
\texttt{action\_space\_size} & 2 & 4 & 4 & 4 & 4 & 5\\
\hline
\texttt{state\_space\_size} &462 & 81 & 900 & 400 & 400 & \~100k\\
\hline
\end{tabular}
\caption{Environment Parameters}
\label{tab:envparams}
\label{table1}
\end{table*}
}
\section{Experiments}
We demonstrate the viability of our approach with four case studies. The algorithm used to compute an inductive $\epsilon$-upper bound of $\beta_{\mathcal M}$ is Interval Iteration \cite{HADDAD2018111}, which is simple in our case as the end components of the MDPs corresponding to the environments are trivial. We use PPO \cite{SPPO} as an RL algorithm to find an optimal policy in the shield. We demonstrate the viability of our approach with five case studies of increasing complexity. For each case study, we compare the safety and the cumulative reward given at each epoch by unshielded PPO \cite{SPPO}, PPO-shield (our approach), PPO-Lagrangian \cite{Achiam2019BenchmarkingSE}, a combination of a lagrangian approach and PPO, and CPO \cite{DBLP:conf/icml/AchiamHTA17}. We use Omnisafe \cite{ji2023omnisafeinfrastructureacceleratingsafe} for the implementation of PPO-Lagrangian and CPO.

\subsection{Environment descriptions}

We provide descriptions for each of our testing environments below. For the gridworld environments, the agent has access to four actions in every state (except for the terminal one), which are $\{\textit{left}, \textit{right}, \textit{up}, \textit{down}\}$. Every action carries a probability \texttt{random\_action\_probability} of choosing randomly, in a uniform manner, another direction. For example, the action $\textit{left}$ makes the agent go left with probability $1-\texttt{random\_action\_probability}$, and the agent goes right, up, and down with remaining probability $\texttt{random\_action\_probability}/3$. Furthermore, safety in all the environments is defined as avoiding the unsafe states with probability at least $1-\texttt{safety\_bound}$.

Table~\ref{tab:envparams} details the parameters of each of our environments including this random probability, the maximum episode length, the total number of interactions (or timesteps) and the safety bound. We also provide illustrations of the relevant environments in Figure~\ref{fig:envs}.

\paragraph{Media streaming}
The agent is tasked with managing a data buffer. The data buffer has size $20$, with packets leaving the data buffer according to a Bernoulli process with rate $\mu_{\textit{out}} = 0.7$. The agent has two actions $A = \{\textit{fast}, \textit{slow}\}$ which fill the data buffer with new packets according to a Bernoulli process with rates $\mu_{\textit{fast}}=0.9$ and $\mu_{\textit{slow}} = 0.1$ respectively. The goal is to minimise the outage time: if the data buffer is empty, the agent receives a reward of $-1$ and $0$ otherwise. The state space is augmented with a cost $c$ which corresponds to the number of times the action \textit{fast} is used. The unsafe states are all the states corresponding to a total number of \textit{fast} actions used above the threshold $ C = \lfloor\texttt{episode\_length}/2\rfloor$. Thus, the agent must avoid using more that $C$ \textit{fast} actions with high probability. A similar environment has been considered in \cite{DBLP:conf/nips/BuraHKSC22}.
\paragraph{Colour bomb gridworld v1}
The agent operates in a $9\times9$ gridworld (see Fig.~\ref{fig:9x9colourbomb}). Upon reaching a coloured zone that is yellow, blue or pink, the agent receives a reward of $+1$ and the episode terminates. Alternatively, when reaching the green or red zones, the agent can choose either to stay inside of them, or to go to any white square that borders. All other states provide a reward of $0$. The unsafe states are the bombs labelled as $B$ states ($S$ denotes the starting state). A similar environment has been used in \cite{ABENTShielding}, albeit with a hard safety constraint instead of a probabilistic one.

\paragraph{Colour bomb gridworld v2}

The agent operates in a $15\times15$ gridworld (see Fig.~\ref{fig:15x15colourbomb}), similar to the previous environment. However, in contrast to the previous environment, the non-green coloured zones that give a reward of $+1$ and terminate the episode are randomised, either at the start of an episode or when the agent enters the green zone.

\paragraph{Bridge crossing (v1 and v2)}
The agent operates in a $20\times20$ gridworld (see Fig.~\ref{fig:bridgecrossing}). The goal is to cross the bridge to the safe terminal yellow states, which provide a reward of $+1$. The unsafe states are the red states (lava), and the agent must thus avoid falling in lava with high probability. The start state is denoted by the green square. Bridge crossing v1 has been used in \cite{DBLP:conf/aaai/MittaH0KKA24}.

\begin{figure}[!t]
    \centering
    \begin{subfigure}[t]{0.47\textwidth}
        \centering
        \includegraphics{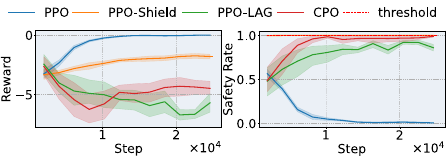}
        \caption{Media streaming}
        \label{fig:mediastreamingresults}
    \end{subfigure}
    \hfill
    \begin{subfigure}[t]{0.47\textwidth}
        \centering
        \includegraphics{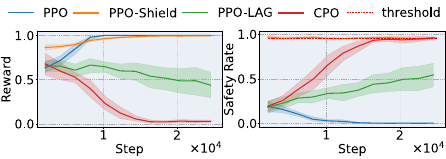}
        \caption{Colour bomb gridworld v1}
        \label{fig:colourbombresults}
    \end{subfigure}
    \hfill
    \begin{subfigure}[t]{0.47\textwidth}
        \centering
        \includegraphics{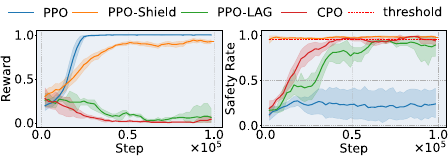}
        \caption{Colour bomb gridworld v2}
        \label{fig:colourbombv2results}
    \end{subfigure}
    \hfill
    \begin{subfigure}[t]{0.47\textwidth}
        \centering
        \includegraphics{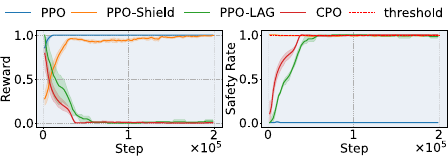}
        \caption{Bridge crossing v1}
        \label{fig:bridgecrossingresults}
    \end{subfigure}
    \hfill
    \begin{subfigure}[t]{0.47\textwidth}
        \centering
        \includegraphics{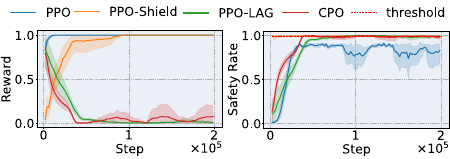}
        \caption{Bridge crossing v2}
        \label{fig:bridgecrossingv2results}
    \end{subfigure}
    \hfill
    \begin{subfigure}[t]{0.47\textwidth}
    \centering
    \includegraphics{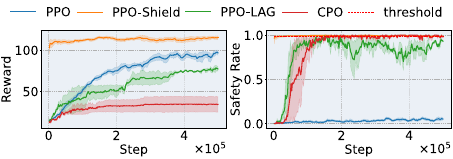}
        \caption{Pacman}
        \label{fig:pacmanresults}
    \end{subfigure}
    \hfill
    \caption{Learning curves}
    \label{fig:learningcurves}
\end{figure}

\paragraph{Pacman} We also consider a $15\times 19$ pacman environment inspired by \cite{racaniere2017imagination}, with one ghost, and collectible coins (+1 reward) in every position (no food). Taking in to consideration all possible locations and directions of the ghost and the agent, and the locations of the coins the state space is combinatorially large, although similar to \cite{ABENTShielding}, we can leverage a \emph{safety abstraction} of the environment (ignoring the dynamics of the coins) for efficient interval iteration. We note that even with the safety abstraction the total number of states exceeds 100k, demonstrating that our approach is still feasible for large state spaces. The goal is to collect as many coins while avoiding the (unsafe) ghosts for the duration of the episode.

\subsection{Results}

Figure \ref{fig:learningcurves} presents the results of our experiments. In every environment, PPO-shield does indeed guarantee safety throughout training and at test time. In terms of cumulative reward, PPO-shield converges to the expected value of $1$ (or almost $1$) in the Colour bomb gridworld v1 and v2, and Bridge crossing v1 and v2 environments. Furthermore, in the media streaming environment, where there is a trade-off between safety and reward, PPO-shield still improves to an expected strictly negative value. In terms of rate of convergence, PPO-shield converges slightly slower than PPO in every environment except for the Bridge Crossing environments where PPO-shield converges significantly slower and PPO converges immediately. This can be explained by the fact that, if not considering safety, the optimal path for the agent in the Bridge Crossing environments is to go straight up, whereas if considering safety, to get an optimal reward, the agent has to find the path across the bridge in Bridge Crossing v1, and the path that goes around the lava to the right in Bridge Crossing v2, correctly evaluating that the straight-up path is too risky. Thus, in these cases, safety is very restrictive, which may explain the longer convergence time. Overall, the rate of convergence of PPO-shield remains fast, requiring a maximum of 100 000 steps in all of our case studies.

We can also see that PPO-shield significantly outperforms CPO and PPO-Lagrangian in all of the case studies. Even though CPO and PPO-Lagrangian both seem to learn the constraint correctly, neither of them manages to optimize the reward in every single one of our case studies. This might be due to the fact that these algorithms are 
slow to converge when the safety requirement is very restrictive.

\section{Conclusion}
We have developed a shielding approach for Safe RL with probabilistic state-avoidance constraints. We have shown that this approach is theoretically sound, and offers strict safety guarantees. Furthermore, this approach relies on Value Iteration on the safety dynamics of the MDP, which is known to be scalable, and allows to decouple the safety dynamics and the reward dynamics of the MDP, in contrast to Safe RL approaches based on Linear Programming. In addition, our experiments show that our method is viable in practice and can significantly outperform state-of-the-art Safe RL algorithms.

\section{Acknowledgements}
This paper was supported by the EPSRC grant number EP/X015823/1, titled "An abstraction-based technique for Safe Reinforcement Learning".
\bibliography{aaai25}
\newpage
\appendix
\section{Appendix}
\setcounter{theorem}{0}
\setcounter{assumption}{0}
In the following, we let $\mathbf F\mathbf u$ denote $\reach{\mathbf u}$.
\subsection{Proof of Theorem 1}
\begin{theorem}[Safety guarantee in any shield]
    For any memoryless policy $\pi$ in $\Shielded{\mathcal M}$, we have $$\Shielded{\mathcal M}_{\pi}\models \mathbb P_{\leq p}\left( \textbf F\textbf u\right).$$
\end{theorem}
\begin{proof}
    We fix 
    a memoryless policy $\pi$ in $\Shielded{\mathcal M}$. We let $T'$ be the set of all states $(s,q)$ of $\Shielded{\mathcal M}$ labeled by $c$, and we let $S'_{?}=S'\setminus T'$. Furthermore, we let $P'_\pi$ be the mapping from $S'$ to the space of probability measures over $S'$ such that $P'_\pi$ is the transition function of the Markov chain $\Shielded{\mathcal M}_\pi$, \emph{i.e.}, such that $$P'_\pi(s,q)(E)=\int_{a\in A'(s,q)}P'((s,q),a)(E)d\pi((s,q))(a).$$
    
    Further, we let $\textbf A$ be the operator over the set of Borel-measurable mappings from $S'_{?}$ to $[0;1]$ such that $$(\textbf A f)(s,q)=\int_{(s',q')\in S'_{?}} f(s',q') d P'_\pi((s,q))((s',q'))$$ 
    and we let $\textbf b$ be the mapping that associates every $(s,q)\in S'_{?}$ to $P'_\pi((s,q))(T')$, and let $\mathbf{\Gamma}$ be the operator such that $\mathbf{\Gamma}(f)=\textbf A(f)+\textbf b$ for any Borel-measurable mapping $f$ from $S'_?$ to $[0;1]$. In addition, we let $\omega^0(s,q)=0$ for any $(s,q)\in S'_{?}$ and $\omega^{n+1}=\mathbf{\Gamma}(\omega^n)$. 
    
    We show, by induction on $n$, that \begin{equation}\omega^n(s,q)=\text{prob}^{s,q}_{\Shielded{\mathcal M},\pi}(E^{s,q}_n),\label{omega-equation} \end{equation} where $E^{s,q}_n$ is the set of all infinite paths $\zeta=\zeta_0a_0\cdots\zeta_n a_n\cdots$ of $\Shielded{\mathcal M}$ such that $\zeta_0=(s,q)$ and 
    there exists $i\in\{0,\ldots,n\}
    $ with $\zeta_i\in T'$.
    \begin{itemize}
        \item If $n=0$, $E_n$ is the set of all states labeled by $c$, and (\ref{omega-equation}) follows.
        \item Suppose now that for some $n\in\mathbb N$, for any $(s,q)\in S'_?$, we have $\omega^n(s,q)=\text{prob}^{s,q}_{\Shielded{\mathcal M},\pi}(E^{s,q}_n)$. The probability of a path $\zeta=\zeta_0a_0\cdots$ such that $\zeta_0=(s,q)$ being in $E^{s,q}_{n+1}$ (according to probability measure $\text{prob}^{s,q}_{\Shielded{\mathcal M},\pi}$) is the sum of the probability of $\zeta_1a_1\cdots$ being in $E^{\zeta_1}_n$, and of the probability of $\zeta_1$ being in $T'$. However, the probability of $\zeta_1 a_1\cdots\in E^{\zeta_1}_n$ is equal to $(\textbf A\omega^n)(s,q)$ by the induction hypothesis, and the probability of $\zeta_1\in T'$ is equal to $\textbf b(s,q)$ by definition of $\textbf b$. Therefore,  (\ref{omega-equation}) follows.
    \end{itemize}

    Let us now define the relation $\leq$ over mappings $f,g$ from $S'_{?}$ to $[0;1]$ such that $f\leq g$ iff $f(s,q)\leq g(s,q)$ for every $(s,q)\in S'_{?}$, and let $\eta(s,q)=q$ for any $(s,q)\in S'_{?}$. We have that $\mathbf{\Gamma}(\eta)(s,q)$ is equal to $\int_{(s',q')\in S'} q' dP_\pi((s,q))((s',q'))$, which is the same as $$\int_{a\in A'(s,q)}\int_{(s',q')\in S'} q' dP((s,q),a)((s',q'))d\pi((s,q))(a).$$ Thus, since $$\int_{(s',q')\in S'} q' dP((s,q),a)((s',q'))\leq q$$ for any $a\in A(s)$ by definition of $\Shielded{\mathcal M}$, we have $\mathbf{\Gamma}(\eta)\leq \eta$. Furthermore, since integrating preserves the relation $\leq$, then the operator $\mathbf{\Gamma}$ is increasing for $\leq$. Consequently, since $\omega^0\leq \eta$ and since $\mathbf{\Gamma}(\omega^{n})=\omega^{n+1}$, an immediate induction gives $\omega^n(s,q)\leq \eta$ for any $n\in\mathbb N$. Furthermore, from (\ref{omega-equation}), since $\bigcup_{n\in\mathbb N}E_n^{s,q}$ is equal to $$\{\zeta=\zeta_0\cdots\in\paths{\Shielded{\mathcal{M}}}\mid \zeta_0=(s,q)\land\zeta\models \textbf F \mathbf u\}$$ and it is a countable union, we have that for any $(s,q)\in S'_{?}$, $\omega^n(s,q)$ converges to $$\prob{\Shielded{\mathcal M},\pi}{(s,q)}{\zeta \in \paths{\Shielded{\mathcal M}}\mid \zeta\models \mathbf F\mathbf u}$$ as $n$ tends to infinity. The result follows.
\end{proof}
\subsection{Proof of Theorem \ref{theorem-opt}}
\begin{assumption}[Slater's condition]
    There exists a policy $\pi$ in $\mathcal M$ and a number $q<p$ such that 
    $$\mathcal M_{\pi}\models \mathbb P_{\leq q}\left( \textbf F\textbf u\right).$$
\end{assumption}
\begin{theorem}[Optimality-preserving guarantees]
We have the three following properties.
\begin{enumerate}
    \item For any $\epsilon>0$, for any inductive $\epsilon$-upper bound $\beta$ of $\beta_{\mathcal M}$, there exists an optimal, memoryless, and deterministic policy $\pi^\star_\beta$ of $\Shielded{\mathcal M}$.
    \item The policy $\addmemory{\pi^\star_{\beta_{\mathcal M}}}$ is a solution to RCOP.
    \item If Assumption \ref{assupmtion-slater} holds, if $\pi^\star$ is a solution to RCOP, then $$\lim_{\beta\in \setbetas{\mathcal M}, \beta\rightarrow \beta_{\mathcal M}} J\left(\addmemory{\pi^\star_{\beta}}\right)=J\left(\addmemory{\pi^\star_{\beta_{\mathcal M}}}\right).$$
\end{enumerate}
\end{theorem}
\begin{proof}
We show the three properties.

\begin{enumerate}
    \item Since $\Shielded{\mathcal M}$ satisfies the conditions for the lower semi-continuous model (Definition 8.7 of \cite{BSSstochastic}), there exists a memoryless, deterministic, and optimal policy $\pi^\star_\beta$ of $\Shielded{\mathcal M}$ (Corollary 9.17.2 of \cite{BSSstochastic}).
    \item For any policy $\pi$ of $\mathcal M$ such that $$\prob{\mathcal M,\pi}{s_{init}}{\zeta \in \paths{\mathcal M}\mid \zeta\models \textbf F\textbf u} \leq p,$$ we let $\Shielding{\pi}$ be the policy of $\Shieldedtwo{\mathcal{M}}{\beta_{\mathcal M}}$ composed of $\pi$ and of predicting ``safety levels" equal to the probabilities of reaching an unsafe state in $\mathcal M_\pi$, \emph{i.e.}, the policy such that for any path $\zeta=(s_0,q_0)\cdots(s_n,q_n)$ of $\Shieldedtwo{\mathcal{M}}{\beta_{\mathcal M}}$ with $$q_n= \prob{\mathcal M,\pi}{s_n}{\zeta \in \paths{\mathcal M}\mid \zeta\models \textbf F\textbf u},$$ $\Shielding{\pi}(\zeta)$ is defined as $\sum_{v\in V_\alpha^{s_n,q_n}} \lambda_v (\alpha,v)$
 where $$\alpha(s)=\prob{\mathcal M,\pi}{s}{\zeta \in \paths{\mathcal M}\mid \zeta\models \textbf F\textbf u}$$
 and where the $(\lambda_v)_{v\in V_\alpha^{s_n,q_n}}\in [0;1]^{V_\alpha^{s_n,q_n}}$ are such that
 \begin{eqnarray*}
 \sum_{v\in V_\alpha^{s_n,q_n}} \lambda_v=1 &
\text{and} &
 \sum_{v\in V_\alpha^{s_n,q_n}} \lambda_v v=\pi(\zeta).
 \end{eqnarray*}
For any policy $\pi$ of $\mathcal M$, since $\pi_{\beta_{\mathcal{M}}}^\star$ is optimal for $\Shieldedtwo{\mathcal{M}}{\beta_{\mathcal M}}$, we have $J(\pi_{\beta_{\mathcal{M}}}^\star)\geq J(\Shielding{\pi})$. 
Moreover, $J(\pi^\star_{\beta_{\mathcal M}})=J\left(\addmemory{\pi^\star_{\beta_{\mathcal M}}}\right)$, and by definition of the reward function of $\Shieldedtwo{\mathcal{M}}{\beta_{\mathcal M}}$, we have $J(\Shielding{\pi})=J(\pi)$. Thus, we have that $J\left(\addmemory{\pi^\star_{\beta_{\mathcal M}}}\right)\geq J(\pi)$ for any safe policy $\pi$ of $\mathcal M$.

 \item Let $\epsilon>0$. Since $J(\widehat{\pi})=J(\pi)$ for any memoryless $\pi$ policy of $\Shielded{\mathcal M}$, we show that there exists $\eta>0$ such that if $\infnorm{\beta-\beta_{\mathcal M}}\leq\eta$, then $\left|J\left(\pi^\star_{\beta}\right)-J\left(\pi^\star_{\beta_{\mathcal M}}\right)\right|\leq\epsilon$. The fact that $J\left(\pi^\star_{\beta_{\mathcal M}}\right)\geq J\left( \pi^\star_{\beta} \right)$ comes from the definition of $J\left(\pi^\star_{\beta}\right)$. Therefore, it remains to find $\eta>0$ such that, if $\infnorm{\beta-\beta_{\mathcal M}}<\eta$, then $$J\left(\pi^\star_{\beta_{\mathcal M}}\right)\leq J\left( \pi^\star_{\beta} \right)+\epsilon.$$
 We suppose without loss of generality that $\#A(s)=d$ for all states $s$ of $\mathcal M$. 
 
 We first define another policy $\pi_1$ of $\Shielded{\mathcal M}$ such that $$J\left(\pi^\star_{\beta_{\mathcal M}}\right)\leq J\left( \pi_1 \right)+\frac{\epsilon}{3}.$$ We let $\omega=\frac{\epsilon(1-\gamma)^2}{3(d-1)(r_{max}-r_{min})}$, and for any state $s$ of $\mathcal M$, for any mapping $\alpha$ from $S$ to $[0;1]$, we let $a^\alpha$ be the action in $A(s)$ such that $$\sum_{s'\in S} \alpha(s')P(s,a^\alpha,s')=\min_{a\in A(s)} \sum_{s'\in S} \alpha(s')P(s,a,s').$$
 We define $\pi_1$ as a policy such that, for any $(s,q)$, $\pi_1(s,q)=\sum_{v\in V^{s,q}_\alpha }\lambda_v\delta_{(\alpha,v)}$ such that $\delta_{(\alpha,v^\star)}=\pi^\star_{\beta_{\mathcal M}}(s,q)$, and such that $v'=\sum_{v\in V^{s,q}_\alpha }\lambda_v v$ satisfies
 $$\left\{ \begin{array}{cc}
      v'_a=v^\star_a+\sum_{a'\in\{a'\mid v^\star_{a'}\leq\omega\}} v^\star_{a'}& \text{ if } a=a^\alpha \\
      v'_a=0 & \text{ if } v^\star_a\leq\omega \\
      v'_a=v^\star_a & \text{otherwise.}
 \end{array} \right.$$
 Notice that 
 \begin{multline*}
 \sum_{a\in A(s)} v'_a\left(\sum_{s'\in S}P(s,a,s')\alpha(s')\right)\leq\\ \sum_{a\in A(s)} v_a\left(\sum_{s'\in S}P(s,a,s')\alpha(s')\right)+\\\sum_{a\in\{a\mid v^s_a\leq \omega\}} v_a\Bigg(\left(\sum_{s'\in S}P(s,a,s')\alpha(s')\right)-\\ \left(\sum_{s'\in S}P(s,a^\alpha,s')\alpha(s')\right)\Bigg), 
 \end{multline*}
 which is less than or equal to $q$ by definition $a^\alpha$ and since $v\in V^{s,q}_\alpha$. Thus, $v'\in V^{s,q}_\alpha$ and the $\lambda_v$ in the definition of $\pi_1$ are well-defined. Since for any $(s,q)$, the probability measures $\pi_1(s,q)$ and $\pi^\star_{\beta_{\mathcal M}}(s,q)$ select the same actions with probability at least $1-\omega(d-1)$, we have \begin{multline*}
     J(\pi_1)\geq J(\pi^\star_{\beta_{\mathcal M}})-\\\sum_{t\in\mathbb N} \gamma^{t+1}\sum_{t'\in\mathbb N}\gamma^{t'}\omega(d-1)(r_{max}-r_{min})\\=\frac{\gamma\omega(d-1)(r_{max}-r_{min})}{(1-\gamma)^2}=\frac{\epsilon}{3}.
 \end{multline*}
 
We now a define policy $\pi_2$ of such that $J(\pi_2)\geq J(\pi_1)-\frac{2\epsilon}{3}$. We let \begin{equation}
    \lambda=\frac{\epsilon(1-\gamma)^2}{3(r_{max}-r_{min})}\label{equation-lambda}
\end{equation} and we let $M\in\mathbb N$ be such that \begin{equation}
    \frac{\gamma^M}{1-\gamma}(r_{max}-r_{min})\leq \frac{\epsilon}{3}.\label{equation-M}
\end{equation} 
Furthermore, we let $E$ be the (finite) set of all $(s,q)$ such that there exists a path $(s_0,q_0) a_0\cdots (s_n,q_n)$ in $\Shieldedtwo{\mathcal M}{\beta_{\mathcal M}}_{\pi_1}$ such that $n<M$, $(s_0,q_0)=(s_{init},p)$ and $(s,q)=(s_n,q_n)$, and we let $\delta_{min}$ be the minimum of all the $q-\beta_{\mathcal M}(s)$ such that $(s,q)\in E$ and $q-\beta_{\mathcal M}(s)$, which exists because of Assumption \ref{assupmtion-slater}. In addition, we let $h_0,\ldots,h_M$, $\lambda_0,\ldots,\lambda_M$, $\theta_0,\ldots \theta_M$, and $\eta_0,\ldots,\eta_M$ be four non-decreasing sequences such that, for any $n<M$
\begin{gather}
    h_0=\theta_0=\epsilon_0=0\\
    \lambda_0>0\label{eq-lambda0}\\
    h_{n+1}= h_n+\frac{2}{\omega} \frac{\lambda_n}{\lambda_{n+1}}\label{eq-hn}\\
    \frac{\lambda_n}{\lambda_{n+1}}\leq \frac{\delta_{min}\omega}{32}\label{eq-ratiolambda}\\
    h_M\leq \frac{\delta_{min}}{4}\label{eq-hM}\\
    \lambda_{M}\leq \min\left\{\frac{1}{4},\lambda\right\}\label{lambdan-to-lambda}\\
    \theta_n=(h_{n+1}-h_{n})\lambda_{n+1}
\end{gather}
and we let \begin{equation}\eta=\min\left\{\theta_1\frac{\delta_{min}\omega}{8},\frac{\lambda_1\omega\delta_{min}^2}{32}\right\}.\label{def-eta}\end{equation} It is easy to check that such four sequences exist, as we only need to take $\lambda_0$ and $\frac{\lambda_n}{\lambda_{n+1}}$ sufficiently small, and the fact that $\eta>0$ comes from (\ref{eq-lambda0}). For any $s\in S$, for any $n\leq M$, we also let $\theta_n(s,q)$ be the number equal to $\min \left\{\frac{\lambda_n\delta_{min}}{4},\theta_n\right\}$  if  $$\sum_{a\in A(s)}v^\star_a\sum_{s'\in S} P(s,a,s')(\alpha^\star(s')-\beta_{\mathcal M}(s'))\leq \frac{\delta_{min}\omega}{4}$$ 
   and equal to $\theta_n$ otherwise, where $\{(\alpha^\star,v^\star)\}$ is the support of the Dirac distribution $\pi_1(s,q)$.
Finally, we let $\beta$ be such that $\infnorm{\beta-\beta_{\mathcal M}}\leq \eta$, and we let $\pi_2$ be a policy  of $\Shielded{\mathcal M}$ such that, for any path $\zeta=(s_0,q_0) a_0 \cdots (s_n,q_n)$ of $\Shieldedtwo{\mathcal M}{\beta_{\mathcal M}}_{\pi_1}$ with $(s_0,q_0)=(s_{init},p)$ and $n\leq M-1$, if we let $(\alpha^\star,v^\star)$ denote the support of the Dirac distribution $\pi_1(s_n,q_n)$, $a_{safe}$ denote an action of $A(s)$ such that $$\sum_{s'\in S}P(s,a_{safe},s')\beta_{\mathcal M}(s')=\beta_{\mathcal M}(s),$$ 
and $t_i$ denote $$\left(q_i-\beta_{\mathcal M}(s_i)-\frac{\delta_{min}\omega}{8}\right)\theta_i(s_i,q_i),$$
\begin{itemize}
    \item if there exists $i\leq n$ such that $q_i=0$, then if $j$ is the minimal integer $i$ that has this property, we have 
    \begin{multline}
        \pi_2\bigg((s_0,q_0)\\ a_0 (s_1,q_1-t_1) \cdots   a_{j-2}(s_{j-1},q_{j-1}-t_{j-1}) \\a_{j-1} (s_j,q_j+\epsilon)\cdots a_{n-1}(s_n,q_n+\epsilon)\Bigg)=\\(v^\star,\alpha^\star+\epsilon)\label{pi2-def1}
    \end{multline}
    \item otherwise, \begin{multline}
\pi_2\Bigg((s_0,q_0)\\ a_0 (s_1,q_1-t_1) \cdots  a_{n-2}\left(s_{n-1},q_{n-1}-t_{n-1}\right) \\ a_{n-1}\left(s_n,q_n-t_n\right)\Bigg)=\\\Big(\lambda_n \chi_{a_{safe}}+(1-\lambda_n) v^\star,\\ \theta_n(s_n,q_n)\left(\beta_{\mathcal M}+\frac{\delta_{min}\omega}{8}\right)+(1-\theta_n(s_n,q_n))\alpha^\star\Big).\label{pi2-def2}
\end{multline}
\end{itemize}
Notice that the ``safety levels" output by $\pi_2$ in (\ref{pi2-def1}) and (\ref{pi2-def2}) are above $\beta_{\mathcal M}+\eta$ by (\ref{def-eta}), the fact that $\theta$ is non-decreasing, and the definition of $\theta_n(s,q)$. The fact that $J(\pi_2)\geq J(\pi_1)-\frac{2\epsilon}{3}$ comes from the fact if $n<M$ and $(s_0,q_0)=(s_{init},p)$, the distributions $\pi_2((s_0,q_0)a_0\cdots (s_n,q_n))$ and $\pi_1(s_n,q_n)$ are the same on a set of measure $1-\lambda_{n}$ by definition of $\pi_2$, from the fact that $\lambda_n$ is non-decreasing, from (\ref{lambdan-to-lambda}), from (\ref{equation-lambda}), and from (\ref{equation-M}). 

It remains to show that $\pi_2$ is well-defined as a policy of $\Shielded{\mathcal M}$, \emph{i.e.}
 that for any finite path $\zeta=(s_0,q_0) a_0\cdots (s_n,q_n)$ of $\Shieldedtwo{\mathcal M}{\beta_{\mathcal M}}_{\pi_1}$ with $n<M$ and $q_i>\beta_{\mathcal M}(s_i)$ for any $0\leq i\leq n$, if we let $\{(\alpha^\star,v^\star)\}$ be the support of the Dirac distribution $\pi
_1(s_n,q_n)$, if we let $a_{safe}$ be the action such that $\beta_{\mathcal M}(s_n)=\sum_{s'\in S} P(s,a_{safe},s') \beta_{\mathcal M}(s')$, and if we let $P(s_n,a^\star,s')=\sum_{a\in A(s)}v^\star_a P(s_n,a,s')$, we have 
 \begin{multline}
     (1-\lambda_n)\sum_{s'\in S} \Bigg[q_n-P(s_n,a^\star,s')\\\Bigg( \theta_n(s_n,q_n)\left(\beta_{\mathcal M}(s')+\frac{\delta_{min}\omega}{8}\right)+ \\(1-\theta_n(s_n,q_n))\alpha^\star\Bigg)\Bigg]+\\\lambda_n\Bigg[q_n-\sum_{s'\in S} P(s_n,a_{safe},s') \\\Bigg( \theta_n(s_n,q_n)\left(\beta_{\mathcal M}(s')+\frac{\delta_{min}\omega}{8}\right)+\\(1-\theta_n(s_n,q_n))\alpha^\star\Bigg)\Bigg]\leq q_n-t_n.\label{eq-pi2-inshield}
 \end{multline}
To show (\ref{eq-pi2-inshield}), we first transform (\ref{eq-pi2-inshield}) as the following inequation that implies (\ref{eq-pi2-inshield})
 \begin{multline}
     (1-\lambda_n)\theta_n(s_n,q_n)\sum_{s'\in S} \Bigg[P(s_n,a^\star,s') \\\Bigg(\alpha^\star(s')-\beta_{\mathcal M}(s')-\frac{\delta_{min}\omega}{8}\Bigg)\Bigg]+\\
     \lambda_n(1-\theta_n(s_n,q_n))\sum_{s'\in S} \Bigg[P(s_n,a_{safe},s')\\ \left(\beta_{\mathcal M}(s')+\delta_n-\alpha^\star(s')\right)\Bigg]\\
     -t_n\lambda_n \geq 0,\label{eq-pi2-inshield-transformed}
 \end{multline}
 where $\delta_n=q_n-\beta_{\mathcal M}(s_n)$.
 
It thus remains to show (\ref{eq-pi2-inshield-transformed}), and to do so, we distinguish the two following cases.
\begin{itemize}
    \item Suppose that $$\sum_{s'\in S} P(s,a^\star,s')(\alpha^\star(s')-\beta_{\mathcal M}(s'))\leq \frac{\delta_{min}\omega}{4}.$$
    Then by definition of $\pi_1$, we have that for all $s'\in S$, $|\alpha^\star(s')-\beta_{\mathcal M}(s')|\leq \frac{\delta_{min}}{4}$. Therefore, we have \begin{multline}
     (1-\lambda_n)\theta_n(s_n,q_n)\sum_{s'\in S} \Bigg[P(s_n,a^\star,s') \\\Bigg(\alpha^\star(s')-\beta_{\mathcal M}(s')-\frac{\delta_{min}\omega}{8}\Bigg)\Bigg]\\\geq -\theta_n(s_n,q_n)\frac{\delta_{min}\omega}{8}\\\geq -\frac{\lambda_n\delta_{min}^2\omega}{16},\label{eq-17}
     \end{multline}
     \begin{multline}
         \lambda_n(1-\theta_n(s_n,q_n))\sum_{s'\in S} \Bigg[P(s_n,a_{safe},s')\\ \left(\beta_{\mathcal M}(s')+\delta_n-\alpha^\star(s')\right)\Bigg]\geq \lambda_n\frac{3\delta_{min}}{8}\label{eq-18}
     \end{multline}
     and \begin{equation}-t_n\lambda_n\geq -\frac{\lambda_n^2\delta_{min}}{4}.\label{eq-tnlambdan}\end{equation} Equation (\ref{eq-pi2-inshield-transformed}) is thus a consequence of (\ref{eq-17}), (\ref{eq-18}), (\ref{eq-tnlambdan}) and (\ref{lambdan-to-lambda}).
     \item Suppose now that $$\sum_{s'\in S} P(s,a^\star,s')(\alpha^\star(s')-\beta_{\mathcal M}(s'))= \frac{K\delta_{min}\omega}{4},$$ with $K>1$. Then, by definition of $\pi_1$, we have that for all $s'\in S$, $|\alpha^\star(s')-\beta_{\mathcal M}(s')|\leq \frac{K\delta_{min}}{4}$. Therefore, we have from (\ref{eq-hn}) and (\ref{eq-ratiolambda})
     \begin{multline}
     (1-\lambda_n)\theta_n(s_n,q_n)\sum_{s'\in S} \Bigg[P(s_n,a^\star,s') \\\Bigg(\alpha^\star(s')-\beta_{\mathcal M}(s')-\frac{\delta_{min}\omega}{8}\Bigg)\Bigg]\\\geq \frac{(h_{n+1}-h_n)\lambda_{n+1}}{2}\delta_{min}\omega\left(\frac{K}{4}-\frac{1}{8}\right)\\\geq \lambda_n\delta_{min}\left(\frac{K}{4}-\frac{1}{8}\right)
     \label{eq-17(2)}
     \end{multline}
     \begin{multline}
         \lambda_n(1-\theta_n(s_n,q_n))\sum_{s'\in S} \Bigg[P(s_n,a_{safe},s')\\ \left(\beta_{\mathcal M}(s')+\delta_n-\alpha^\star(s')\right)\Bigg]\geq -\lambda_n\delta_{min}\left(\frac{K}{4}-\frac{1}{4}\right)\label{eq-18(2)}
     \end{multline}
     and from (\ref{eq-hM})
     \begin{multline}
         t_n\lambda_n\geq -(h_{n+1}-h_n)\lambda_{n+1}\lambda_n\\\geq -\frac{\delta_{min}\lambda_n}{32}\label{eq-19(2)}
     \end{multline}
     Equation (\ref{eq-pi2-inshield-transformed}) is thus a consequence of (\ref{eq-17(2)}), (\ref{eq-18(2)}), (\ref{eq-19(2)}).
\end{itemize}
\end{enumerate}
\end{proof}

\subsection{Additional experiments}
We ran additional experiments to compare our approach with PPO-Lagrangian \cite{Achiam2019BenchmarkingSE} and CPO \cite{DBLP:conf/icml/AchiamHTA17}. Since for small cost limits these algorithms seem to struggle, we changed the parameter \texttt{safety\_bound} of our case studies to \texttt{0.5}. Due to compute constraints, results are averaged over 3 independent runs (rather than the usual 10). For the environments Bridge Crossing v1 and v2, Colour Bomb v2 and Media Steaming v2, PPO-Lagrangian and CPO fail to converge within a million steps. Figure~\ref{fig:additionalexperiments} presents the results for the environments Colour Bomb v1 and Media Streaming.
\begin{figure}[!h]
    \centering
    \hfill
    \begin{subfigure}[t]{0.5\textwidth}
        \centering
        \includegraphics[width=0.45\linewidth]{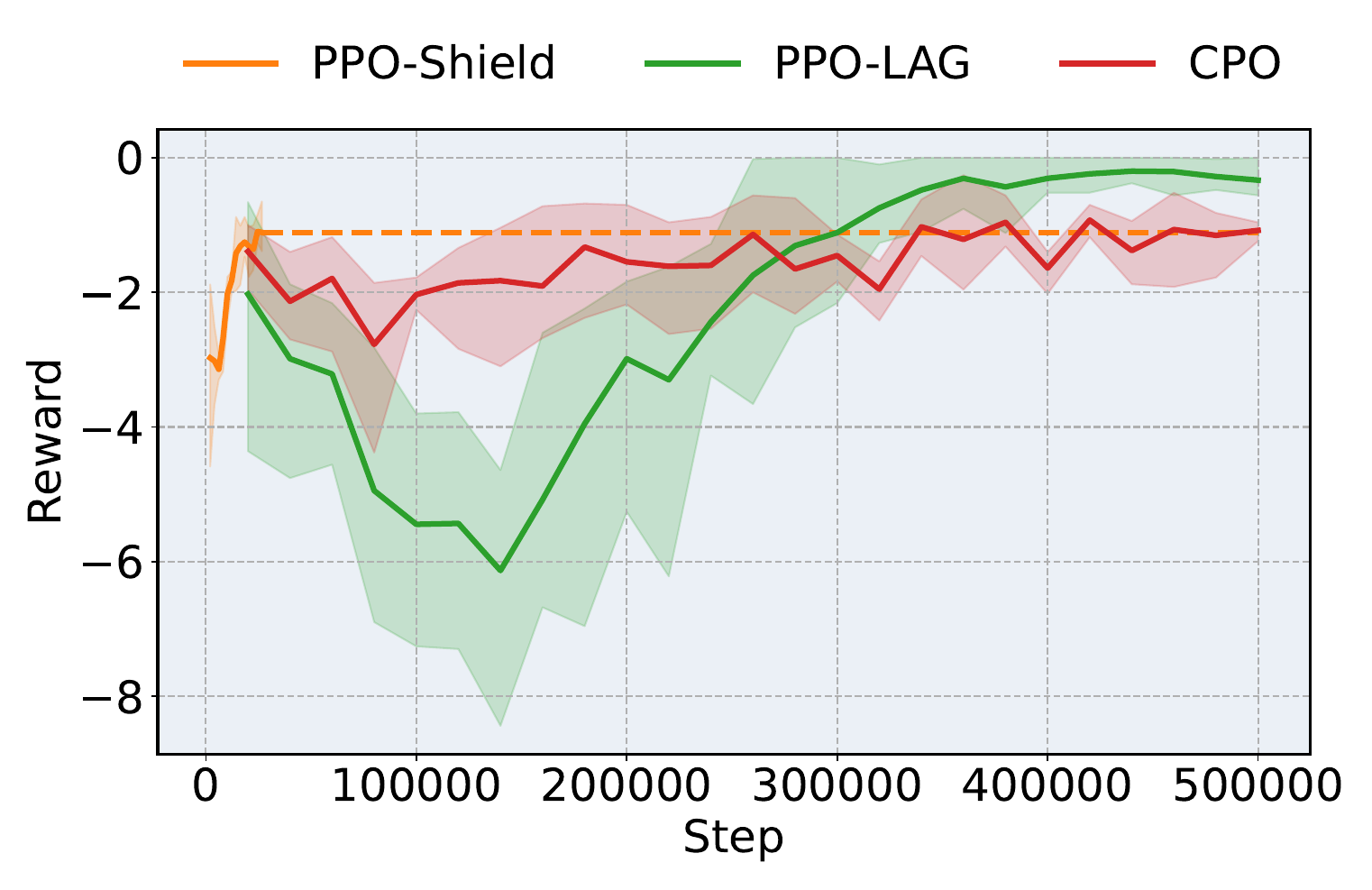}
        \includegraphics[width=0.45\linewidth]{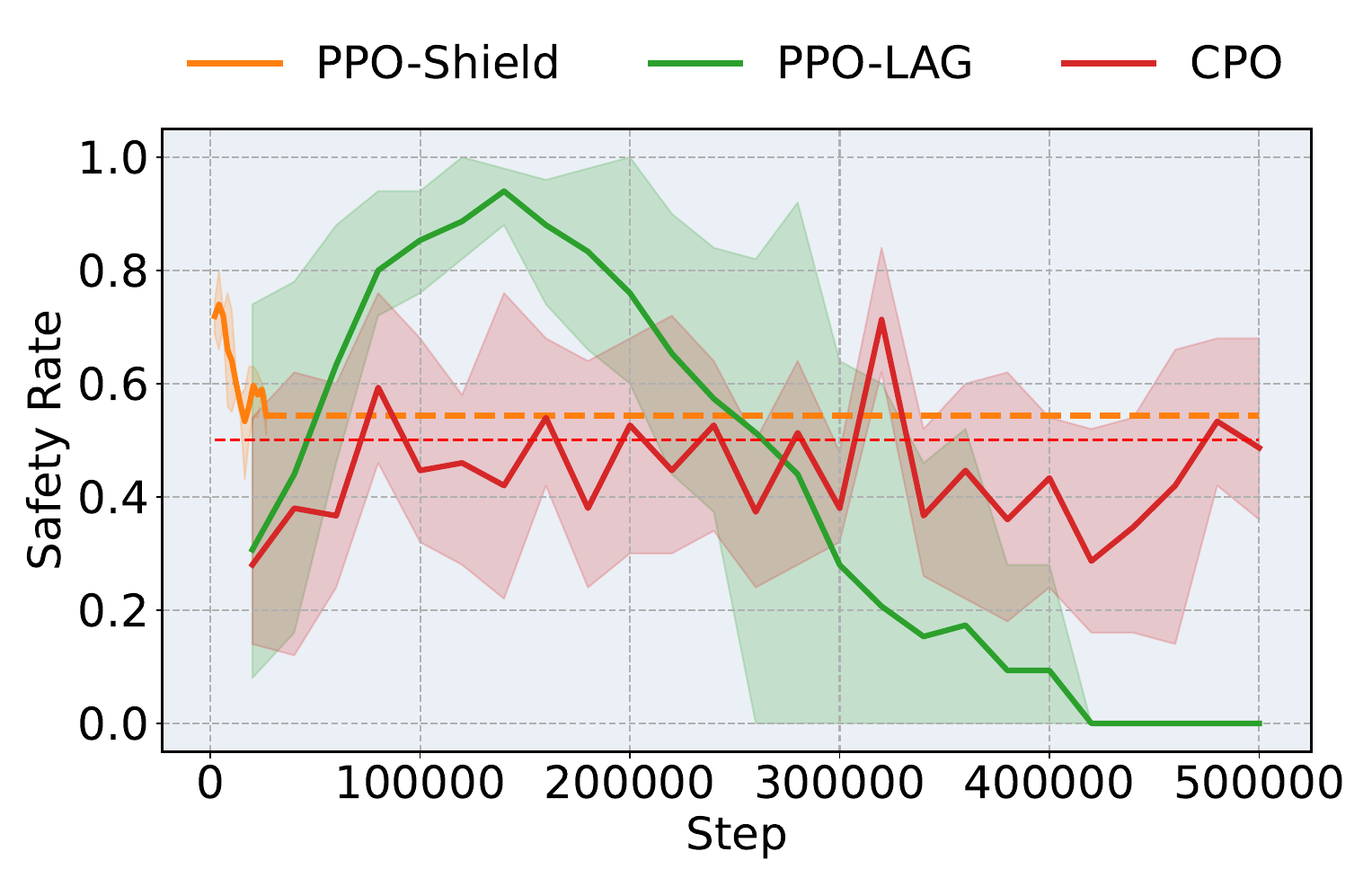}
        \caption{Media streaming}
        \label{fig:mediastreamingresults}
    \end{subfigure}
    \hfill
    \begin{subfigure}[t]{0.5\textwidth}
        \centering
        \includegraphics[width=0.45\linewidth]{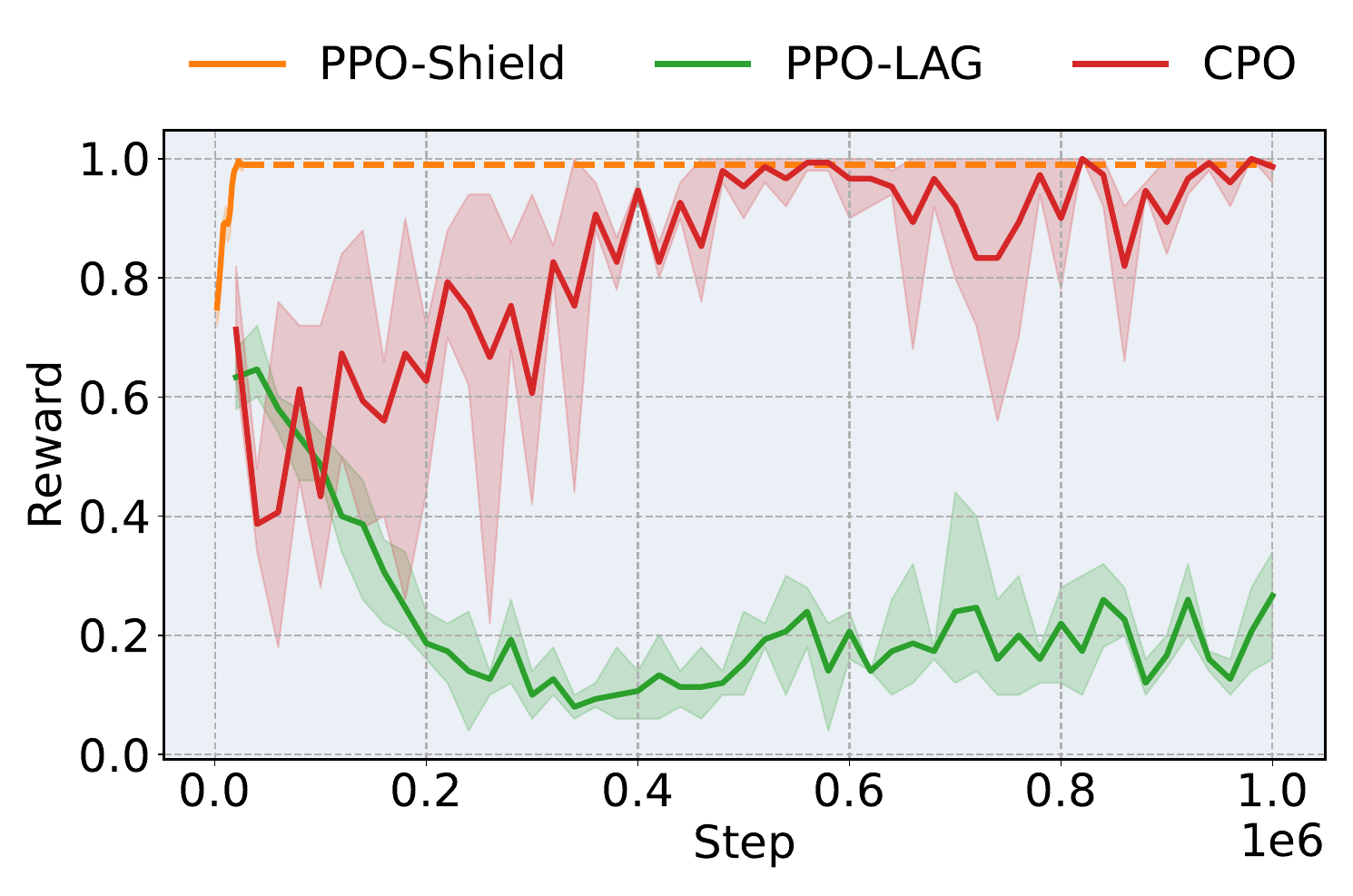}
        \includegraphics[width=0.45\linewidth]{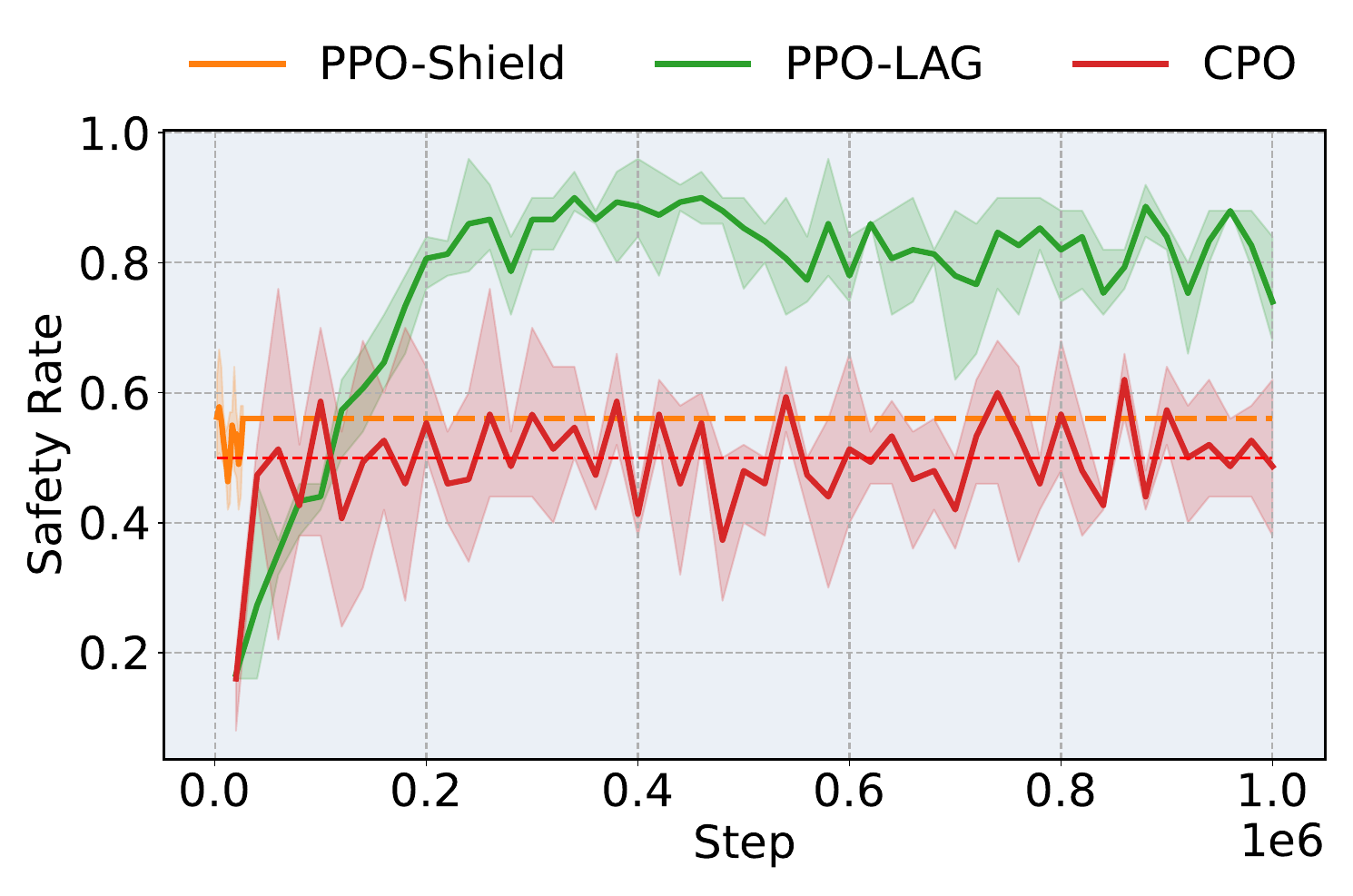}
        \caption{Colour bomb gridworld v1}
        \label{fig:colourbombresults}
    \end{subfigure}
    \hfill
    \caption{Learning curves for additional experiments}
    \label{fig:additionalexperiments}
\end{figure}

For the Media Streaming and the Colour Bomb v1 environments, we can see that CPO converges to the optimal policy roughly within the cost limit of 0.5. However, it converges much more slowly than PPO-Shield, even though these environments are quite simple. Unfortunately, PPO-Lag fails to converge in either environment, likely due to slow convergence of the dual variable.

\subsection{Hyperparameters}
For our implementation of PPO and PPO-Shield we used the default hyperparameters provided by stable baselines3 \cite{stable-baselines3}: \texttt{lr = 0.0003}, \texttt{n\_steps=2048}, \texttt{batch\_size=64}, \texttt{n\_epochs=10}, \texttt{gae\_lambda=0.95}, \texttt{clip=0.2}, \texttt{max\_grad\_norm=0.5}, \texttt{ent\_coef=0.0} and \texttt{vf\_coef=0.5}.

For PPO-Lagrangian and CPO (main paper) we used comparable hyperparameters where applicable: \texttt{lr = 0.0003}, \texttt{n\_steps=2048}, \texttt{batch\_size=64}, \texttt{n\_epochs=10}, \texttt{gae\_=0.95}, \texttt{gae\_\_cost=0.95}, \texttt{clip=0.2}, \texttt{max\_grad\_norm=0.5} and \texttt{ent\_coef=0.0}. For the different environments in the main paper, we used \texttt{cost\_limit=0.05} for colour bomb (v2), \texttt{cost\_limit=0.01} for bridge crossing (v2) and \texttt{cost\_limit=0.01} for media streaming which correspond to the \texttt{safety\_bounds} used for each environment.

For the additional experiments we updated some of the hyperparameters for longer run training (for PPO-Lagrangian and CPO): \texttt{n\_steps=20000}, \texttt{batch\_size=128}, \texttt{n\_epochs=40}, \texttt{max\_grad\_norm=40.0}. Finally, in all experiments for PPO-Lagranian we set the \texttt{lagrangian\_multiplier\_init=10.0}.

\end{document}